\documentclass[draftcls,onecolumn,peerreview]{IEEEtran}

\newcounter{mytempeqncnt}
\newtheorem{thm}{Theorem}

\newtheorem{defntn}{Definition}

\newenvironment{proof}[1][Proof]{\begin{trivlist}
\item[\hskip \labelsep {\bfseries #1}]}{\end{trivlist}}

\newcommand{\qed}{\nobreak \ifvmode \relax \else
      \ifdim\lastskip<1.5em \hskip-\lastskip
      \hskip1.5em plus0em minus0.5em \fi \nobreak
      \vrule height0.75em width0.5em depth0.25em\fi}

\ifCLASSINFOpdf
  \usepackage[pdftex]{graphicx}
	\usepackage{epstopdf}
  \graphicspath{{./}{Figs/}}
  \DeclareGraphicsExtensions{.pdf,.jpeg,.png}
\else
\fi
\usepackage{amssymb}
\usepackage{amsmath}
\usepackage{mathtools}
\usepackage{enumerate}
\usepackage{subcaption}
\usepackage{pgfplots}
\usepackage[caption=false,font=footnotesize]{subfig}
\usepackage{undertilde}
\usepackage{verbatim}
\usetikzlibrary{arrows,calc}
\usepackage{relsize}

\usepackage{algpseudocode}
\usepackage{algorithm2e}
\newcommand{\indepid}[1]{\ensuremath{\stackrel{\text{#1}}{\sim}}}

\usepackage{array}
\newcolumntype{P}[1]{>{\centering\arraybackslash}p{#1}}
\newcolumntype{M}[1]{>{\centering\arraybackslash}m{#1}}

\renewcommand{\algorithmicrequire}{\textbf{Input:}}
\renewcommand{\algorithmicensure}{\textbf{Output:}}

\newcommand{\finalcells}[2]{%
  \begingroup\sbox0{\begin{minipage}{3cm}\raggedright#1\end{minipage}}%
  \sbox2{\begin{minipage}{3cm}\raggedright#2\end{minipage}}%
  \xdef\finalheight{\the\dimexpr\ht0+\dp0+\smallskipamount\relax}%
  \xdef\finalheightB{\the\dimexpr\ht2+\dp2+\smallskipamount\relax}%
  \ifdim\finalheightB>\finalheight
    \global\let\finalheight\finalheightB
  \fi\endgroup
  \begin{minipage}[t][\finalheight][t]{3cm}\raggedright#1\end{minipage}&
  \begin{minipage}[t][\finalheight][t]{3cm}\raggedright#2\end{minipage}}

\begin{document}
\title{ConfidentCare: A Clinical Decision Support System for Personalized Breast Cancer Screening}

\author{Ahmed~M.~Alaa,~\IEEEmembership{Member,~IEEE}, Kyeong H. Moon, William Hsu,~\IEEEmembership{Member,~IEEE},~and Mihaela~van~der~Schaar,~\IEEEmembership{Fellow,~IEEE}
\thanks{A. M. Alaa, K. H. Moon, and M. van der Schaar are with the Department of Electrical Engineering, University of California Los Angeles, UCLA, Los Angeles, CA, 90024, USA (e-mail: ahmedmalaa@ucla.edu, mihaela@ee.ucla.edu). This work was supported by the NSF.}
\thanks{W. Hsu is with the Department of Radiological Sciences, UCLA, Los Angeles, CA 90024, USA (email: willhsu@mii.ucla.edu).} 
}
\markboth{}%
{Alaa \MakeLowercase{\textit{et al.}}: ConfidentCare: A Clinical Decision Support System for Personalized Breast Cancer Screening}
\maketitle
\begin{abstract}
Breast cancer screening policies attempt to achieve timely diagnosis by the regular screening of apparently healthy women. Various clinical decisions are needed to manage the screening process; those include: selecting the screening tests for a woman to take, interpreting the test outcomes, and deciding whether or not a woman should be referred to a diagnostic test. Such decisions are currently guided by clinical practice guidelines (CPGs), which represent a ``one-size-fits-all" approach that are designed to work well on average for a population, without guaranteeing that it will work well uniformly over that population. Since the risks and benefits of screening are functions of each patients features, {\it personalized screening policies} that are tailored to the features of individuals are needed in order to ensure that the right tests are recommended to the right woman. In order to address this issue, we present {\it ConfidentCare}: a computer-aided clinical decision support system that learns a personalized screening policy from the electronic health record (EHR) data. ConfidentCare operates by recognizing clusters of ``similar" patients, and learning the ``best" screening policy to adopt for each cluster. A cluster of patients is a set of patients with similar features (e.g. age, breast density, family history, etc.), and the screening policy is a set of guidelines on what actions to recommend for a woman given her features and screening test scores. ConfidentCare utilizes an iterative algorithm that applies $K$-means clustering to the women's feature space, followed by learning an active classifier (decision tree) for every cluster. The algorithm ensures that the policy adopted for every cluster of patients satisfies a predefined accuracy requirement with a high level of confidence. We show that our algorithm outperforms the current CPGs in terms of cost-efficiency and false positive rates.     
\end{abstract}
\IEEEpeerreviewmaketitle{}
\section{Introduction}
\IEEEPARstart {P}{ersonalized} medicine is a new healthcare paradigm that aims to move beyond the current ``one-size-fits-all" approach to medicine and, instead, takes into account the features and traits of individual patients: their genes, micro-biomes, environments, and lifestyles \cite{refx1}-\cite{refxx3}. Vast attention has been recently dedicated to research in personalized medicine that builds on data science and machine learning techniques in order to customize healthcare policies. For instance, the White House has led the ``precision medicine initiative" \cite{refx3}, which is scheduled for discussion in the American Association for the Advancement of Science annual meeting for the year 2016 \cite{refx4}. Breast cancer screening is one example for a healthcare process that can potentially benefit from personalization. Screening is carried out in order to diagnose a woman with no apparent symptoms in a timely manner \cite{ref01}-\cite{ref03}. However, the screening process entails both benefits and costs that can differ from one patient to another \cite{ref04}, which signals the need for personalized screening policies that balance such benefits and costs in a customized manner. 

In this paper, we present ConfidentCare: a clinical decision support system (CDSS) that is capable of learning and implementing a personalized screening policy for breast cancer. The personalized screening policy is learned from data in the electronic health record (EHR), and is aimed to issue recommendations for different women with different features on which when should they take screening tests, which specific tests to take, and in what sequence. ConfidentCare discovers subgroups of ``similar" patients from the EHR data, and learns how to construct a screening policy that will work well for each subgroup with a high level of confidence. Our approach can provide significant gains in terms of both the cost-efficiency, and the accuracy of the screening process as compared to other ``one-size-fits-all" approaches adopted by current clinical practice guidelines (CPGs) that apply the same policy on all patients.    

\subsection{Breast cancer screening and the need for personalization}
 
While breast cancer screening is believed to reduce mortality rates \cite{ref03}, it is associated with the risks of ``overscreening", which leads to unnecessary costs, and ``overdiagnosis", which corresponds to false positive diagnoses that lead the patients to receive unnecessary treatments \cite{ref04}. While different patients have different levels of risks for developing breast cancer \cite{ref05}-\cite{ref09}; different tests have different monetary costs, and different levels of accuracy that depend on the features of the patient \cite{ref10}; common CPGs are aimed at populations, and are not typically tailored to specific individuals or significant subgroups \cite{ref11}-\cite{ref12x2}. 

Being designed to work well on ``average" for a population of patients, following CPGs may lead to overscreening or overdiagnosis for specific subgroups of patients, such as young women at a high risk of developing breast cancer, or healthy older women who may have a relatively longer expected lifespan \cite{ref13}. Moreover, some screening tests may work well for some patients, but not for others (e.g. a mammogram test will exhibit low accuracy for patients with high breast density \cite{ref10}), which can either lead to ``overdiagnosis" or poor tumor detection performance. Migrating from the ``one-size-fits-all" screening and diagnosis policies adopted by CPGs to more individualized policies that recognizes and approaches different subgroups of patients is the essence of applying the personalized medicine paradigm to the breast cancer clinical environment \cite{ref10}, \cite{ref13}-\cite{ref16}.  

\subsection{Contributions}
 
ConfidentCare is a computer-aided clinical decision support system that assists clinicians in making decisions on which (sequence of) screening tests a woman should take given her features. ConfidentCare resorts to the realm of supervised learning in order to learn a personalized screening policy that is tailored to subgroups of patients. In particular, the system recognizes different subgroups of patients, learns the policy that fits each subgroup, and prompts recommendations for screening tests and clinical decisions that if followed will lead to a desired accuracy requirement with a desired level of confidence. 

Fig. \ref{figu2} offers a system-level illustration for ConfidentCare\footnote{We will revisit this figure and give a more detailed explanation for the system components in the next Section}. The system operates in two stages: an offline stage in which it learns from the EHR data how to cluster patients, and what policy to follow for every cluster, and an execution stage in which it applies learned policy to every woman by first matching her with the closest cluster of patients in the EHR, and then approach her with the policy associated with that cluster. The main features of ConfidentCare are: 
\begin{itemize}
\item ConfidentCare discovers a set of patients' subgroups. Given certain accuracy requirements and confidence levels set by the clinicians, ConfidentCare ensures that every subgroup of patients would experience a diagnostic accuracy, and a confidence level on that accuracy, that meets these requirements. Thus, unlike CPGs that perform well only on average, ConfidentCare ensures that performance is reasonable for every discovered subgroups of patients.
\item ConfidentCare ensures cost-efficiency, i.e. patients are not overscreened, and the sequence of recommended screening tests minimizes the screening costs. 
\end{itemize}
We show that ConfidentCare can improve the screening cost-efficiency when compared with CPGs, can offer performance guarantees for individual subgroups of patients with a desired level of confidence, and outperforms the ``one-size-fits-all" approaches in terms of the accuracy of clinical decisions. Moreover, we show that ConfidentCare can achieve a finer granularity in its learned policy with respect to the patients feature space when it is provided with more training data. Our results emphasize the value of personalization in breast cancer clinical environments, and represent a first step towards individualizing breast cancer screening, diagnosis and treatment.

\subsection{Related works}

\subsubsection{Personalized (precision) medicine} 
While medical studies investigated the feasibility, potential and impact of applying the concepts of personalized medicine in the breast cancer clinical environments \cite{refx1}-\cite{refxx3}, \cite{ref10}-\cite{ref18}, \cite{ref21}\cite{ref21x}, none of these works provided specific tools or methods for building a personalized healthcare environment. For instance, in \cite{ref10}, it has been shown that CPGs, which recommend screening tests only based on the age ranges, such as the European Society for Medical Oncology (ESMO) CPG and the American Cancer Society (ACS) CPG, are not cost-efficient for many subgroups of patients, where cost-efficiency was measured in terms of ``costs per quality-adjusted life-year", and the authors recommended that screening should be personalized on the basis of a patient's age, breast density, history of breast biopsy, and the family history of breast cancer. Similar results were portrayed in other medical studies \cite{ref16}-\cite{ref18}, all suggesting that personalization using dimensions other than the age can yield more cost efficiency. 

Personalizing breast cancer screening is envisioned to not only improve the cost-efficiency of the process, but also to improve the diagnostic accuracy. This is because current CPGs do not consider the individual features of a woman when recommending screening tests; thus decisions on which a woman needs to take an additional screening tests, or proceed to a diagnostic test (biopsies) are not tailored to the woman's individual features. Therefore, false negative diagnoses rates reported by clinicians who follow CPGs reflect the average accuracy over all the population of patients, but CPGs give no guarantee that the diagnostic accuracy and the associated confidence levels of their guidelines are reasonable for every subgroup of ``similar" patients \cite{ref19}\cite{ref20}; such subgroups can be significantly different in their traits and hence may require being dealt with via different screening and diagnosis policies. 
\\
\subsubsection{Dynamic treatment regimes} Perhaps the work that relates most to this paper is that on Dynamic treatment regimes (DTRs) \cite{ref22x1}-\cite{ref22x5}. A DTR is typically a sequence of decision rules, with one rule per stage of clinical intervention, where each rule maps up-to-date patient information to a recommended treatment \cite{ref22x1}. DTRs are constructed via reinforcement learning techniques, such as Q-learning, where the goal is to find an ``optimal treatment policy": a sequential mapping of the patient's information to recommended treatments that would maximize the patient's long term reward. However, these works profoundly differ from the setting we consider in the following aspects: DTRs are only focused on recommending treatments and do not consider screening and diagnoses; cost-efficiency is not considered in the design of DTR policies since they only consider the ``value of information" in recommending treatments; and finally, while confidence measures can be computed for policies in DTRs \cite{ref22x3}, the policies themselves are not designed in a way that guarantees to the clinician a certain level of reliability for every subgroup of patients.  
\\
\subsubsection{Active classification for medical diagnosis} Screening and diagnostic clinical decisions typically involve ``purchasing costly information" for the patients, which relates to the paradigm of active learning \cite{ref25}-\cite{ref31}. We note that in our setting, clinicians ``purchase" costly features of the patients rather than purchasing unobserved labels, which makes our setting profoundly different from the conventional active learning framework \cite{ref25}-\cite{ref22z4}. Classification problems in which some features are costly are referred to as ``active classification" \cite{ref24}, or ``active sensing" \cite{ref30x}. Such problems have been addressed in the context of medical diagnosis in \cite{ref24}-\cite{ref31}, but all these works correspond to solving an unconstrained optimization problem that targets the whole population, for which no personalized accuracy or confidence guarantees can be claimed. Table \ref{tab1} positions our paper to the existing literature with respect to various aspects.    
\begin{table}[t!]
\captionsetup{font= small}
\caption{Comparison against existing literature}
\begin{center}
 \centering
    \begin{tabular}{||M{1.5cm}||M{1.75cm}|M{2cm}|M{1.25cm}|} \hline
        {\bf Method} & {\bf Personalization} & {\bf Accuracy and confidence guarantees} & {\bf Cost-efficiency}  \\ \hline \hline
				    DTRs & Yes & No &  No   \\ \hline
						Active classification & No & No & Yes   \\ \hline
						ConfidentCare & Yes & Yes & Yes \\ 
    \hline
    \end{tabular}
\end{center}
\label{tab1}
\end{table}

The rest of the paper is organized as follows. In Section II, we present the system components and the problem formulation for designing personalized screening policies. Next, in Section III, we propose the ConfidentCare algorithm. In Section IV, we carry out various experiments using a dataset collected at the UCLA medical center in order to highlight the advantages of ConfidentCare. Finally, in Section V, we draw our conclusions.   

\section{ConfidentCare: system components and operation}
\subsection{System operation}
ConfidentCare is a computer-aided clinical decision support system that learns a personalized screening policy from the EHR data. By a ``personalized screening policy" we mean: a procedure for recommending an action for the clinician to take based on the individual features of the patient, and the outcomes of the screening tests taken by that patient. An action can be: letting the patient take an additional screening test, proceed to a diagnostic test (e.g. biopsy), or just recommend a regular follow-up.       
 
The tasks that ConfidentCare carries out can be summarized as follows:
\begin{itemize}
\item {\bf Discover the granularity of the patient's population}: The system is provided with training data from the EHR that summarizes previous experiences of patients in terms of the screening tests they took, their test results, and their diagnoses. From such data, ConfidentCare recognizes different {\it subgroups} or {\it clusters} of patients who are similar in their features and can be approached using the same screening policy. 

\item {\bf Learn the best policy for each subgroup of patients}: Having discovered the distinct subgroups of patients from the training data, ConfidentCare finds the best screening policy for each of these subgroups; by a ``best" policy we mean: a policy that minimizes the screening costs while maintaining a desired level of diagnostic accuracy, with a high level of confidence that is set by the clinicians. The more training data provided to ConfidentCare, the more ``granular" is the learned policy in the sense that more subgroups of patients can be discovered, and thus the extent of personalization and precision would increase consequently.

\item {\bf Identify the incoming patients' subgroups and execute their personalized policies}: After being trained, ConfidentCare handles an incoming patient by observing her features, identifying the subgroup to which she belongs, and decides the appropriate screening policy that needs to be followed for her.    
\end{itemize}

ConfidentCare can be thought of as an algorithm that stratifies the pool of patients into clusters, and automatically generates multiple CPGs, one for each cluster, in order to issue the best customized guidelines to follow for each cluster. The algorithm ensures that the accuracy of clinical decisions for each cluster satisfy a certain requirement with a certain confidence level. 

\subsection{Idiosyncrasies of the breast cancer clinical environment}

Patients' features fall into two categories: {\it personal features}, and {\it screening features}. Personal features are observable at no cost, and are accessible without the need for taking any screening tests, for that they are provided by the patient herself via a questionnaire, etc. The personal features include numerical and categorical features such as: age, age at menarche, number of previous biopsies, breast density, age at first child birth, and the family history \cite{ref10}. 

Screening tests reveal another set of costly features for the patient, which we call: the screening features. The screening features comprise the radiological assessment of breast images, usually encoded in the form of BI-RADS (Breast Imaging Report and Data System) scores \cite{ref19}. The BI-RADS scores take values from the set $\{1,2,3,4A,4B,4C,5,6\},$ the interpretation of which is given in Table II. BI-RADS scores of 3 or above are usually associated with followup tests or biopsy. The descriptions of all the personal and screening features are shown in Table III. 

ConfidentCare considers three possible multimedia-based screening tests in the screening stage, which represent three different imaging modalities: mammogram (MG), ultrasound (US), and magnetic resonance imaging (MRI). Every screening test is associated with different costs and risks, which are functions of the patients' personal features. We consider a generic cost function that incorporates both the misclassification costs in addition to the monetary costs (the detailed cost model is provided in the next subsection) \cite{ref18}. Other screening features can also include genetic ones, yet we do not consider these in this paper since such features are not revealed by the screening tests under consideration. However, ConfidentCare algorithm together with the theoretical framework tackled in this section can handle any generic class of features and tests, including genetic tests.

ConfidentCare recommends an action upon observing the outcome of a specific screening test. The actions can either be: recommend a regular (1 year) followup, recommend a diagnostic test (biopsy), or an intermediate recommendation for an additional (costly) screening test (short-ter followup). The final action recommended by the screening policy is either to proceed to a diagnostic test, or to take a regular followup (screening) test after 1 or 2 years. The accuracy measures that we adopt in this paper are: the false positive rate (FPR) and the false negative rate (FNR), which are defined as follows: the FPR is the probability that a patient with a negative true diagnosis (benign or no tumor) is recommended to proceed to a diagnostic test, whereas the FNR is the probability that a patient with a positive true diagnosis (malignant tumor) is recommended to take a regular followup screening test \cite{ref22}.  

\begin{table}[t!]
\captionsetup{font= small}
\caption{BI-RADS scores interpretation}
\begin{center}
 \centering
    \begin{tabular}{||M{2cm}||M{4.5cm}|} \hline
        {\bf Score} & {\bf Interpretation} \\ \hline \hline
				    0 & Incomplete.   \\ \hline
						1 & Negative.  \\ \hline
						2 & Benign. \\ \hline
						3 & Probably benign. \\ \hline
						4A & Low suspicion for malignancy. \\ \hline
						4B & Intermediate suspicion of malignancy. \\ \hline
						4C & Moderate concern. \\ \hline
						5 & Highly suggestive of malignancy. \\ \hline
						6 & Known biopsy -- proven malignancy. \\ 
    \hline
    \end{tabular}
\end{center}
\label{tab1}
\end{table}  
\begin{table}[t!]
\captionsetup{font= small}
\caption{Personal and screening features}
\begin{center}
 \centering
    \begin{tabular}{||M{2cm}||M{4.5cm}|} \hline
        {\bf Personal feature} & {\bf Description and range of values} \\ \hline \hline
					  Age information & Age at screening test time-age at menarche-age at first child birth.   \\ \hline
						Family history & Number of first degree relatives who developed breast cancer (First degree relatives are: mother, sister, and daughter).\\ \hline
						Number of previous biopsies & An integer number of biopsies. \\ \hline
						Breast density & Described by four categories: 
						\begin{itemize}
						\item {\bf Category 1}: The breast is almost entirely fat (fibrous and glandular tissue $<25\%$).
						\item {\bf Category 2}: There are scattered fibro-glandular densities (fibrous and glandular tissue $25\%$ to 50$\%$).
						\item {\bf Category 3}: The breast tissue is heterogeneously dense (fibrous and glandular tissue $50\%$ to $75\%$).
						\item {\bf Category 4}: The breast tissue is extremely dense (fibrous and glandular tissue $>75\%$).
						\end{itemize}
						    \\ 
    \end{tabular}
		    \begin{tabular}{||M{2cm}||M{4.5cm}|} \hline
        {\bf Screening features} & {\bf Description} \\ \hline \hline
				    MG BI-RADS & Radiological assessment of the mammogram imaging.   \\ \hline
						US BI-RADS & Radiological assessment of the ultrasound test.  \\ \hline
						MRI BI-RADS & Radiological assessment of the MRI test. \\ 
    \hline
    \end{tabular}
\end{center}
\label{tab2}
\end{table} 

\subsection{System components}

ConfidentCare is required to deal with the environment specified above and carry out the three tasks mentioned earlier, which are: discovering the granularity of the patients' population, learning the appropriate policies for each subgroup of patients, and handling incoming patients by executing the learned, personalized policy that best matches their observed features and traits. 

In the following, we describe the ConfidentCare algorithm, which implements those tasks using supervised learning. The algorithm requires the following inputs from the clinician: 
\begin{itemize}      
\item A training set comprising a set of patients with their associated features, screening tests taken, and their true diagnoses.
\item A restrictions on the maximum tolerable FNR.
\item A desired confidence level on the FNR in the diagnoses issued by the system.
\end{itemize}
Provided by the inputs above, ConfidentCare operates through two basic stages:
\begin{itemize}      
\item {\bf Offline policy construction stage:} Given the training data and all the system inputs, ConfidentCare implements an iterative algorithm to cluster the patients' personal feature space, and then learns a separate {\it active classifier} for each cluster of patients. Each active classifier associated with a cluster of patients is designed such that it minimizes the overall screening costs, and meets the FNR and confidence requirements. The algorithm runs iteratively until it maximizes the number of patient clusters for which there exist active classifiers that can guarantee the performance and confidence requirements set by the clinician, thereby ensuring the maximum level of personalization, i.e. ensure that the space of all patients' personal features is segmented into the finer possible set of partitions, where the performance requirements hold for each of such partitions.     
\\
\item {\bf Policy execution stage:} Having learned a policy based on the training data, ConfidentCare executes the policy by observing the personal features of an incoming patient, associate her with a cluster (and consequently, an already learned active classifier), and then the classifier handles the patient by recommending screening tests and observing the test outcomes, until a final action is recommended.       
\end{itemize}
 
Fig. \ref{figu2} illustrates the components and operation of ConfidentCare. In the {\it offline policy construction stage}, ConfidentCare is provided with training data from the EHR, the maximum tolerable FNR, and the desired level of confidence. ConfidentCare runs an iterative algorithm that clusters the patients' personal feature space, and learns the best active classifier (the most cost-efficient classifier that meets the FNR accuracy and confidence requirements) for each cluster. In the {\it policy execution stage}, ConfidentCare observes the personal features of the incoming patient, associates her with a patients cluster, and then recommends a sequence of screening tests to that patient until it issues a final recommendation.

For instance, assume that the set of personal features are given by a tuple {\it (Age, breast density, number of first degree relatives with breast cancer)}. A patient with a personal features vector (55, 40$\%$,0) is approached by ConfidentCare. The system associates the patient with a certain cluster of patients that it has learned from the EHR data. Let the best policy for screening  patients in that cluster, as computed by ConfidentCare, is to start with mammogram. If the clinician followed such a recommendation, ConfidentCare observed the mammogram BI-RADS score, say a score of 1, and then it decides to issue a final recommendation for a regular followup. If the BI-RADS score was higher, say a score of 4A, then the system recommends an additional imaging test, e.g. an MRI, and then observes the BI-RADS score of the MRI before issuing further recommendations. The process proceeds until a final recommendation is issued. 

\section{The personalized screening policy design problem}

ConfidentCare uses supervised learning to learn a personalized screening policy from the EHR. In this subsection, we formally present the learning model under consideration.
\\ 
\subsubsection{Patients' features}
Let $\mathcal{X}_{d}$, $\mathcal{X}_{s},$ and $\mathcal{Y}$ be three spaces, where $\mathcal{X}_{d}$ is the patients' $d$-dimensional personal feature space, $\mathcal{X}_{s} = \mathcal{B}^{s}$ is the $s$-dimensional space of all screening features, where $\mathcal{B} = \{1,2,3,4A,4B,4C,5,6\}$, and $\mathcal{Y}$ is the space of all possible diagnoses, i.e. $\mathcal{Y} = \{0,1\},$ where $0$ corresponds to a {\it negative} diagnosis, and $1$ corresponds to a {\it positive} diagnosis. The patients' feature space is $(d+s)$-dimensional and is given by $\mathcal{X} = \mathcal{X}_{d} \times \mathcal{X}_{s}$. Each instance in the feature space is a $(d+s)$-dimensional vector ${\bf x} = ({\bf x}_{d}, {\bf x}_{s}) \in \mathcal{X}, {\bf x}_{d} \in \mathcal{X}_{d}, {\bf x}_{s} \in \mathcal{X}_{s},$ the entries of which correspond to the personal and screening features listed in Table III, and are drawn from an unknown stationary distribution $\mathcal{D}$ on $\mathcal{X} \times \mathcal{Y}$, i.e. $({\bf x}, y) \sim \mathcal{D}$, where $y \in \mathcal{Y}$, and $\mathcal{D}_{x}$ is the marginal distribution of the patients' features, i.e. ${\bf x} \sim \mathcal{D}_{x}$. The set of $s$ available tests is denoted by $\mathcal{T}$, where $\left|\mathcal{T}\right| = s$.

The personal features are accessible by ConfidentCare with no cost, whereas the screening features are costly, for that the patient needs to take screening tests to reveal their values. Initially, the entries of ${\bf x}_{s}$ are blocked, i.e. they are all set to an unspecified value $\left<*\right>$, and they are observable only whenever the corresponding screening tests are taken, and their costs are paid. We denote the space of all possible screening test observations as $\mathcal{X}^{*}_{s} = \left\{\mathcal{B}, \left<*\right>\right\}^{s}$. ConfidentCare issues recommendations and decisions based on both the fully observed personal features ${\bf x}_{d}$, and a partially observed version of ${\bf x}_{s}$, which we denote as ${\bf x}_{s}^{*} \in \mathcal{X}^{*}_{s}.$ The screening feature vector ${\bf x}_{s}$ can indeed be fully observed, but this would be the case only if all the screening tests were carried out for a specific patient.   

In order to clarify the different types of features and their observability, consider the following illustrative example. Assume that we only have two personal features: the age and the number of first degree relatives who developed breast cancer, whereas we have three screening tests $\mathcal{T} = \{\mbox{MG}, \mbox{MRI}, \mbox{US}\}$. That is, we have that $d = 2$ and $s = 3$. Initially, ConfidentCare only observes the personal features, e.g. observing a feature vector $(42, 1, \left<*\right>, \left<*\right>, \left<*\right>)$ means that the patient's age is 42 years, she has one first degree relative with breast cancer, and she took no screening tests. Based on the learned policy, ConfidentCare then decides which test should the patient take. For instance, if the policy decides that the patient should take a mammogram test, then the feature vector can then be updated to be $(42, 1, 2, \left<*\right>, \left<*\right>)$, which means that the BI-RADS score of the mammogram is 2. ConfidentCare can then decide what action should be recommended given that the BI-RADS score of the mammogram is 2: classify the patient as one who needs to proceed to a diagnostic test, or classify the patient as one who just needs to take a regular followup test in a 1 year period, or request an additional screening test result in order to be able to issue a confident classification for the patient.

\subsubsection{Active classification}
The process described in the previous subsection is a typical active classification process: a classifier aims to issue either a positive or a negative diagnosis (biopsy or regular followup) for patients based on their costly features (test outcomes). Such a classifier is active in the sense that it can query the clinician for costly feature information rather than passively dealing with a given chunk of data \cite{ref24}. This setting should not be confused with conventional {\it active learning}, where labels (and not features) are the costly piece of information which the classifier may need to purchase \cite{ref25}\cite{ref26}.
In the following, we formally define an {\it active classifier}.   
\begin{defntn}
{\bf (Active classifier)} An active classifier is a hypothesis (function)
\[h: \mathcal{X}^{*}_{s} \rightarrow \mathcal{Y} \cup \mathcal{T}. \,\,\,\,\,\,\,\, \IEEEQEDhere\] 
\end{defntn} 
Thus, the active classifier either recommends a test in $\mathcal{T}$, or issues a final recommendation $y \in \mathcal{Y}$, where $y=1$ corresponds to recommending a biopsy (positive screening test result) and $y = 0$ is recommending a regular followup (negative screening test result), given the current, partially observed screening feature vector ${\bf x}^{*}_{s} \in \mathcal{X}^{*}_{s}$. Whenever a test is taken, the screening feature vector is updated, based upon which the classifier either issues a new recommendation. 

For instance, the range of the function $h$ in our setting can be $\{0,1,\mbox{MG},\mbox{MRI},\mbox{US}\}$, i.e. $\mathcal{Y} = \{0,1\}$ and $\mathcal{T} = \{\mbox{MG},\mbox{MRI},\mbox{US}\}$. If $h({\bf x}^{*}_{s}) = 0$ (or $1$), then the classifier issues -with high confidence on the accuracy- a final recommendation for a biopsy or a regular followup for the patient with a screening feature vector ${\bf x}^{*}_{s} \in \mathcal{X}^{*}_{s}$, whereas if $h({\bf x}^{*}_{s}) = \mbox{MG},$ then the classifier recommends the patient with a screening feature vector ${\bf x}^{*}_{s}$ to take a mammogram test. Note that if $h((\left<*\right>, \left<*\right>, \left<*\right>)) = 0,$ then the classifier recommends no tests for any patient.  
\begin{figure}[t!]
    \centering
    \includegraphics[width=3.5 in]{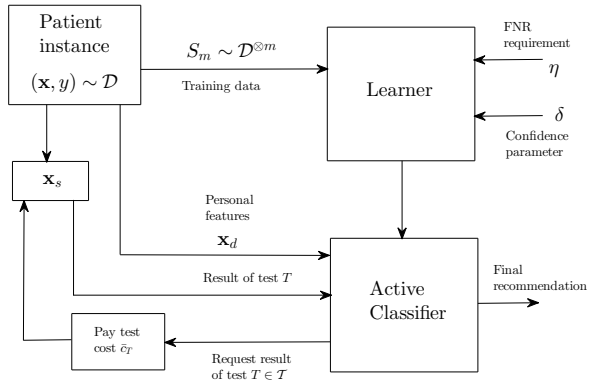}     
    \caption{Framework for the active classifier construction and operation.}
		\label{actvclass}
\end{figure} 

\subsubsection{Designing active classifiers}
Designing an active classifier for the breast cancer screening and diagnosis problem under consideration cannot rely on conventional loss functions, such as the $0-1$ loss function. This is because the classification problem involves costly decision making under uncertainty, and different types of diagnostic errors (false negatives and false positives) have very different consequences. Hence, our notion of learning needs to be {\it decision-theoretic}, and new objective functions and learning algorithms need to be defined and formulated.

We use an {\it inductive bias} approach for designing the active classifier; we restrict our learning algorithm to pick one hypothesis $h$ from a specific hypothesis class $\mathcal{H}$. That is, we compensate our lack of knowledge of the stationary distribution $\mathcal{D}$ by inducing a prior knowledge on the set of possible hypothesis that the learning algorithm can output: a common approach for designing {\it agnostic} learners \cite{ref27}. Unlike the conventional supervised learning paradigm which picks a hypothesis that minimizes a loss function, we will design a learning algorithm that picks a hypothesis from $\mathcal{H}$, such that the overall cost of screening is minimized, while maintaining the FNR to be below a predefined threshold, with a desired level of confidence; a common design objective for breast cancer clinical systems \cite{ref20}. The screening cost involves both the monetary costs of the screening tests, as well as the {\it misclassification cost} reflected by the FPR. The FNR experienced by the patients when using an active classifier $h$ is given by
\begin{equation}      
\mbox{FNR}(h) = \mathbb{P}\left(h({\bf x}^{*}_{s}) = 0 \left|h({\bf x}^{*}_{s}) \in \mathcal{Y}, y = 1\right.\right),
\label{eqtn1}
\end{equation}
whereas the FPR is given by
\begin{equation}      
\mbox{FPR}(h) = \mathbb{P}\left(h({\bf x}^{*}_{s}) = 1 \left|h({\bf x}^{*}_{s}) \in \mathcal{Y}, y = 0\right.\right).
\label{eqtn2}
\end{equation}
That is, the FNR is the probability that classifier $h$ recommends a regular followup (outputs a $0$) for a screening feature vector ${\bf x}_{s}$, when the patient takes all the recommended tests, given that the true diagnosis was $1$, whereas the FPR is the probability that the classifier recommends a biopsy (outputs a $1$) when the true diagnosis is $0$. Both types of error are very different in terms of their implications, and one can easily see that the FNR is more crucial, since it corresponds to misdiagnosing a patient with breast cancer as being healthy \cite{ref21}. Thus, the system must impose restrictions on the maximum tolerable FNR. On the other hand, the FPR is considered as a misclassification cost that we aim at minimizing given a constraint on the FNR \cite{ref18}. 

Now we define the screening cost function. Let $c_{T}$ be the monetary cost of test $T \in \mathcal{T}$, which is the same for all patients, and let $\bar{c}_{T}$ be the normalized monetary cost of test $T$, given by $\bar{c}_{T} = \frac{c_{T}}{\sum_{T^{'} \in \mathcal{T}} c_{T^{'}}}$. Let $\bar{c}(h({\bf x}_{s}))$ be the total (normalized) monetary test costs that classifier $h$ will pay in order to reach a final recommendation for a patient with screening feature vector ${\bf x}_{s}$. The average monetary cost of a hypothesis $h$ is denoted as $\bar{c}(h)$, and is given by $\bar{c}(h) = \mathbb{E}\left[\bar{c}(h({\bf x}_{s}))\right],$ where the expectation is taken over the randomness of the screening test results. To illustrate how the cost of a hypothesis is computed, consider the following example. Let the normalized costs of MG, US, and MRI be 0.1, 0.2 and 0.7 respectively. Initially, the classifier observes ${\bf x}^{*}_{s} = \left(\left<*\right>, \left<*\right>, \left<*\right>\right).$ Assume a hypothesis $h_{1}$ and a patient with a screening features vector ${\bf x}_{s} = \left(3,1,1\right)$. The hypothesis $h_{1}$ has the following functional form: $h_{1}(\left(\left<*\right>, \left<*\right>, \left<*\right>\right)) = \mbox{MG},$ i.e. it initially recommends a mammogram for every patient, $h_{1}(\left(3, \left<*\right>, \left<*\right>\right)) = \mbox{MRI},$ and $h_{1}(\left(3, 1, \left<*\right>\right)) = 0.$ Hence, using $h_{1}$, the screening cost is 0.8. Let $h_{2}$ be another hypothesis with $h_{2}(\left(\left<*\right>, \left<*\right>, \left<*\right>\right)) = \mbox{MG},$ $h_{2}(\left(3, \left<*\right>, \left<*\right>\right)) = 0.$ In this case, we have that $\bar{c}(h_{2}) = 0.1$, which is less than $\bar{c}(h_{1}) = 0.8$, yet it is clear that $h_{2}$ has a higher risk for a false negative diagnosis.

Let $C(h)$ be the {\it cost function} for hypothesis $h$, which incorporates both the average monetary costs and the average misclassification costs incurred by $h$. Formally, the cost function is defined as        
\begin{equation}      
C(h) = \gamma \, \mbox{FPR}(h) + (1-\gamma) \, \bar{c}(h), 
\label{eqtn3}
\end{equation}
where $\gamma \in [0,1]$ is a parameter that balances the importance of the misclassification costs compared to the monetary cost, i.e. $\gamma = 0$ means that ConfidentCare builds the classifiers by solely minimizing monetary costs, whereas $\gamma = 1$ means that ConfidentCare cares only about the misclassification costs. An optimal active classifier is denoted by $h^{*}$, and is the one that solves the following optimization problem
\begin{equation}
\begin{aligned}
& \underset{h \in \mathcal{H}}{\text{min}}
& &  C(h) \\ 
& \text{{\it s.t.}}
& & \mbox{FNR}(h) \leq \eta.
\end{aligned} 
\label{eqtn4}
\end{equation} 
Obtaining the optimal solution for (\ref{eqtn4}) requires knowledge of the distribution $\mathcal{D}$, in order to compute the average FNR and cost in (\ref{eqtn4}). However, $\mathcal{D}$ is not available for the (agnostic) learner. Instead, the learner relies on a size-$m$ training sample $S_{m}=({\bf x}_{i},y_{i})_{i \in [m]},$ with $S_{m} \indepid{i.i.d} \mathcal{D}^{\otimes m}$, where $\mathcal{D}^{\otimes m}$ is the product distribution of the $m$ patient-diagnosis instances $({\bf x}_{i},y_{i})_{i \in [m]}$. The training sample $S_{m}$ feeds a learning algorithm $\mathcal{A}: \mathcal{S}_{m} \rightarrow \mathcal{H},$ where $\mathcal{S}_{m}$ is the space of all possible size-$m$ training samples. The learning algorithm $\mathcal{A}$ simply tries to solve (\ref{eqtn4}) by picking a hypothesis in $\mathcal{H}$ based only on the observed training sample $S_{m}$, and without knowing the underlying distribution $\mathcal{D}$. Fig. \ref{actvclass} depicts the framework for learning and implementing an active classifier.

\subsubsection{Learnability of active classifiers}
In order to evaluate the learner, and its ability to construct a reasonable solution for (\ref{eqtn4}), we define a variant of the {\it probably approximately correct} criterion for learning active classifiers that minimize the classification costs with a constraint on the FNR (conventional definitions for PAC-learnability can be found in \cite{ref24} and \cite{ref27}). Our problem setting, and our notion of learning depart from conventional supervised learning in that the learner is concerned with finding a feasible, and (almost) optimal solution for a constrained optimization problem, rather than being concerned with minimizing an unconstrained loss function. 

In the following, we define a variant for the notion of PAC-learnability, the {\it probably approximately optimal} (PAO) learnability, of a hypothesis set $\mathcal{H}$ that fits our problem setting.    
\begin{defntn}
{\bf (PAO-learning of active classifiers)} We say that active classifiers drawn from the hypothesis set $\mathcal{H}$ are {\it PAO-learnable} using an algorithm $\mathcal{A}$ if:
\begin{itemize}
\item $\mathcal{H}^{*} = \{h: \forall h \in \mathcal{H}, \mbox{FNR}(h) \leq \eta\} \neq \emptyset$, with $h^{*} = \arg \inf_{h \in \mathcal{H}^{*}} C(h)$.
\item For every $(\epsilon_{c},\epsilon,\delta) \in [0,1]^{3}$, there exists a polynomial function $N_{\mathcal{H}}^{*}(\epsilon, \epsilon_{c}, \delta) = poly(\frac{1}{\epsilon_{c}},\frac{1}{\epsilon},\frac{1}{\delta}),$ such that for every $m \geq N_{\mathcal{H}}^{*}(\epsilon, \epsilon_{c}, \delta),$ we have that   
\[\mathbb{P}_{S_{m}\sim\mathcal{D}^{\otimes m}}\left(C\left(\mathcal{A}\left(S_{m}\right)\right) \geq C(h^{*}) + \epsilon_{c} \right) \leq \delta,\]
\[\mathbb{P}_{S_{m}\sim\mathcal{D}^{\otimes m}}\left(\mbox{FNR}(\mathcal{A}\left(S_{m}\right)) \geq \mbox{FNR}(h^{*}) + \epsilon \right) \leq \delta,\]
where $N_{\mathcal{H}}^{*}(\epsilon, \epsilon_{c}, \delta)$ is the {\it sample complexity} of the classification problem. \,\,\, \IEEEQEDhere
\end{itemize}
\end{defntn} 
PAO-learnability reflects the nature of the learning task of the active classifier; a learning algorithm is ``good" if it picks the hypothesis that, with a probability $1-\delta$, is within an $\epsilon$ from the region of feasible region, and within an $\epsilon_{c}$ from the optimal solution. In that sense, a hypothesis set is PAO-learnable if there exists a learning algorithm that can find, with a certain level of confidence, a probably approximately feasible and optimal solution to (\ref{eqtn4}).  

Note that the sample complexity $N^{*}_{\mathcal{H}}(\epsilon, \epsilon_{c}, \delta)$ does not depend on $\eta$, yet the feasibility of the optimization problem in (\ref{eqtn4}), and hence the learnability of the hypothesis class, depends on both the value of $\eta$ and the hypotheses in $\mathcal{H}$. From a {\it bias-variance decomposition} point of view, one can view $\eta$ as a restriction on the amount of inductive bias a hypothesis set can have with respect to the FNR, whereas $\epsilon$, $\epsilon_{c}$ and $\delta$ are restrictions on the true cost and accuracy estimation errors that the agnostic learner would encounter. The threshold $\eta$ qualifies or disqualifies the whole hypothesis set $\mathcal{H}$ from being a feasible set for learning the active classifier, whereas the tuple $(\epsilon, \epsilon_{c}, \delta)$ decides how many training samples do we need in order to learn a qualified hypothesis set $\mathcal{H}$. The notion of PAO-learnability can be thought of as a decision-theoretic variant of the conventional PAC-learnability, for that the learner is effectively solving a constrained cost-minimization problem. 
    
\subsubsection{Patients feature space partitioning}
ConfidentCare learns a different classifier separately for every subgroup of ``similar" patients, which is the essence of personalization. However, the clustering of patients into subgroups is not an input to the system, but rather a task that it has to carry out; ConfidentCare has to bundle patients into $M$ subgroups, each of which will be assigned a different active classifier that is tailored to the features of the patients in that subgroup. The value of $M$ reflects the level of personalization, i.e. the larger $M$ is, the larger is the number of possible classifiers that are customized for every subgroup. Partitioning the patient's population into subgroups is carried out on the basis of the personal features of the patients; patients are categorized based on their personal, fully observable features.

Let $(\mathcal{X}_{d}, d_{x})$ be a {\it metric space} associated with the personal feature space $\mathcal{X}_{d}$, where $d_{x}$ is a {\it distance metric}, i.e. $d_{x}: \mathcal{X}_{d}\times\mathcal{X}_{d} \rightarrow \mathbb{R}_{+}$. We define an $M$-partitioning $\pi_{M}(\mathcal{X}_{d}, d_{x})$ over the metric space $(\mathcal{X}_{d}, d_{x})$ as a set of disjoint subsets of $\mathcal{X}_{d}$, i.e. $\pi_{M}(\mathcal{X}_{d}, d_{x}) = \{\mathcal{C}_{1}, \mathcal{C}_{2},.\,.\,., \mathcal{C}_{M}\},$ where $\mathcal{C}_{i} \subseteq \mathcal{X}_{d},$ $\bigcup_{i=1}^{M}\mathcal{C}_{i} = \mathcal{X}_{d},$ and $\mathcal{C}_{j} \bigcap \mathcal{C}_{i} = \emptyset, \forall i \neq j$. We define a function $\pi_{M}(\mathcal{X}_{d}, d_{x}; {\bf x}_{d})$ as a map from the patient's personal feature vector ${\bf x}_{d}$ to the index of the partition to which she belongs, i.e. $\pi_{M}(\mathcal{X}_{d}, d_{x}; {\bf x}_{d}) = j$ if ${\bf x}_{d} \in \mathcal{C}_{j}$. 

Each partition is simply a subgroup of patients who are believed to be ``similar", where similarity is quantified by a distance metric. By ``similar" patients, we mean patients who have similar risks of developing breast cancer, and experience similar levels of accuracy for the different screening tests. 

\subsubsection{Personalization and ConfidentCare's optimization problem}
Given a certain partitioning $\pi_{M}(\mathcal{X}_{d}, d_{x})$ of the personal feature space, the task of the learner is to learn an active classifier $h_{j} \in \mathcal{H}$ for each partition $\mathcal{C}_j$, that provides (average) performance guarantees for the patients in that partition if the size of the training set is large enough, i.e. larger than the sample complexity\footnote{Note that the training set $S_{m}$ is drawn from the total population of patients, but each active classifier associated with a certain partition is trained using training instances that belong to that partition only.}. This may not be feasible if the size of the training sample is not large enough in every partition, or if the hypothesis set has no feasible hypothesis that have a true FNR less than $\eta$ for the patients in that partition. The following definition captures the extent of granularity with which ConfidentCare can handle the patient's population.    
\begin{defntn}
{\bf ($M$-personalizable problems)} We say that the problem $(\mathcal{H}, S_{m}, \delta, \epsilon, \epsilon_{c}, \mathcal{D})$ is {\it $M$-personalizable} if there exists an $M$-partitioning $\pi_{M}(\mathcal{X}_{d}, d_{x}),$ such that for every partition $\mathcal{C}_{j} \in \pi_{M}(\mathcal{X}_{d}, d_{x})$, $\mathcal{H}$ is PAO-learnable, and we have that $m_{j} \geq N_{\mathcal{H}}^{*}(\epsilon, \epsilon_{c}, \delta),$ where $m_{j} = \left|\mathcal{S}^{j}_{m}\right|$, and $\mathcal{S}^{j}_{m} = \left|\left\{({\bf x}_{i}, y_{i}): i \in [m], {\bf x}_{i,d} \in \mathcal{C}_{j}\right\}\right|$. \,\,\, \IEEEQEDhere      
\end{defntn}
That is, a problem is $M$-personalizable if $\mathcal{H}$ has a non-empty set of feasible hypotheses for every partition, and the number of training samples in every partition is greater than the sample complexity for learning $\mathcal{H}$. 

ConfidentCare is not provided with a feature space partitioning, but is rather required to construct it, i.e. the system should recognize the maximum number of patient subgroups for which it can construct separate active classifiers that meet the accuracy requirements. Partitioning $\mathcal{X}_{d}$ and designing an active classifier for every partition is equivalent to designing a personalized screening policy. Fig. \ref{fgtree} depicts the envisioned output of ConfidentCare for a 2D personal feature space: the feature space is partitioned into 4 partitions, and with each partition, an active classifier (a decision tree) is associated.
\begin{figure}[t!]
    \centering
    \includegraphics[width=3.5 in]{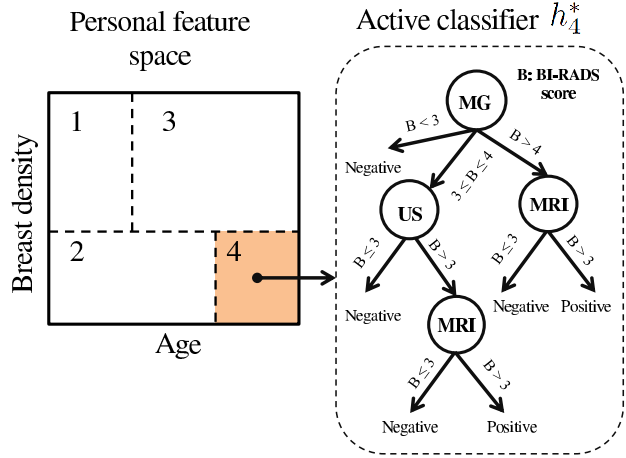}     
    \caption{An exemplary decision tree designed for a specific partition in the personal feature space.}
		\label{fgtree}
\end{figure} 

\begin{figure*}[!t]
\setcounter{mytempeqncnt}{\value{equation}} \setcounter{equation}{4}
\begin{equation}
\Pi = \left\{\pi_{M}(\mathcal{X}_{d}, d_{x}) = \{\mathcal{C}_{1},.\,.\,.,\mathcal{C}_{M}\}\left| \forall \mathcal{C}_{i}\cap \mathcal{C}_{j} = \emptyset, \bigcup_{i=1}^{M} \mathcal{C}_{i} = \mathcal{X}_{d}, \mathcal{C}_{i}\, \forall M \in \{1,2,.\,.\,.,\left|\mathcal{X}_{d}\right|\}\right.\right\}.
\label{eqtn5}
\end{equation}
\setcounter{equation}{\value{mytempeqncnt}+1} \hrulefill{}\vspace*{4pt}
\end{figure*}
\begin{figure*}[!t]
\setcounter{mytempeqncnt}{\value{equation}} \setcounter{equation}{5}
\begin{equation}
\begin{aligned}
& \underset{\pi_{M}(\mathcal{X}_{d}, d_{x}) \in \Pi}{\text{max}}
& &  M \\ 
& \text{{\it s.t.}}
& & (\mathcal{H}, S_m, \epsilon, \delta, \epsilon_{c}, \mathcal{D}) \, \mbox{is $M$-personalizable over $\pi_{M}(\mathcal{X}_{d}, d_{x})$}.
\end{aligned} 
\label{eqtn6}
\end{equation}
\setcounter{equation}{\value{mytempeqncnt}+1} \hrulefill{}\vspace*{4pt}
\end{figure*}
Let $\Pi$ be the set of all possible partitioning maps for the feature space as defined in (\ref{eqtn5}). ConfidentCare aims at maximizing the granularity of its screening policy by partitioning the feature space into the maximum possible number of patient subgroups, such that the active classifier associated with each subgroup of patients ensures that the FNR of this subgroup does not exceed $\eta$, with a confidence level of $1-\delta$. Thus, ConfidentCare is required to solve the optimization problem in (\ref{eqtn6}). Once the optimal partitioning $\pi^{*}_{M}(\mathcal{X}_{d}, d_{x})$ is found by solving (\ref{eqtn6}), the associated cost-optimal classifiers are constructed by solving (\ref{eqtn4}). Designing a screening policy computation algorithm is equivalent to designing a partitioning algorithm $\mathcal{A}^{part}: \mathcal{S}_{m}\rightarrow \Pi$, and a learning algorithm $\mathcal{A}: \mathcal{S}^{j}_{m}\rightarrow \mathcal{H}$. ConfidentCare would operate by running the partitioning algorithm $\mathcal{A}^{part}$ to create a set of partitions of the personal feature space, and then running the learning algorithm $\mathcal{A}$ once for each partition in order to find the appropriate hypothesis for that partition. Ideally, ConfidentCare computes an optimal screening policy if the partitioning found by $\mathcal{A}^{part}$ is a solution to (\ref{eqtn6}). 

\section{ConfidentCare Algorithm: Analysis and design}
We start by analyzing the problem of constructing an optimal screening policy in the next subsection, before we present our proposed algorithm in the following subsection.

\subsection{Optimal screening policies: analysis and technical challenges}
We start exploring the limits of the policy design problem by computing an upper-bound on the maximum level of personalization that can be achieved by any screening policy in the following Theorem. 
\begin{thm}
The maximum level of personalization that can be achieved for the problem $(\mathcal{H}, S_m, \epsilon, \epsilon_{c}, \delta, \mathcal{D})$ is upper-bounded by
\[M^{*} \leq \left\lfloor \frac{m}{N^{*}(\delta, \epsilon, \epsilon_{c})}\right\rfloor,\]
where $M^{*}$ is the solution for (\ref{eqtn6}).        
\end{thm}
\begin{proof}
See Appendix A. \, \IEEEQEDhere  
\end{proof}
Theorem 1 captures the intuitive dependencies of the level of personalization on $m$ and $(\epsilon, \epsilon_{c}, \delta)$. As the training sample size increases, a finer granularity of the screening policy can be achieved, whereas decreasing any of $(\epsilon, \epsilon_{c}, \delta)$ will lead to a coarser policy that has less level of personalization.

While Theorem 1 gives an upper-bound on the possible level of personalization, it does not tell whether such a bound is indeed achievable, i.e. is there a computationally-efficient partitioning algorithm $\mathcal{A}^{part}$, and a learning algorithm $\mathcal{A}$, through which we can we construct an optimal personalized screening policy given a hypothesis set $\mathcal{H}$ and a training sample $S_{m}$? In fact, it can be shown that for any hypothesis class $\mathcal{H}$, the problem of finding the maximum achievable level of personalization in (\ref{eqtn6}) is NP-hard. Thus, there is no efficient polynomial-time algorithm $\mathcal{A}^{part}$ that can find the optimal partitioning of the personal feature space, and hence ConfidentCare has to discover the granularity of the personal feature space via a heuristic algorithm as we will show in the next subsection.

Now we focus our attention to the learning algorithm $\mathcal{A}$. Given that we have applied a heuristic partitioning algorithm $\mathcal{A}^{part}$ to the training data, and obtained a (suboptimal) partitioning $\pi_{M}(\mathcal{X}_{d}, d_{x})$, what hypothesis set $\mathcal{H}$ should we use, and what learning algorithm $\mathcal{A}$ should we chose in order to learn the best active classifier for every partition? In order to answer such a question, we need to select both an appropriate hypothesis set and a corresponding learning algorithm. We start by studying the learnability of a specific class of hypothesis sets.  
\begin{thm}
A finite hypothesis set $\mathcal{H}$, with $\left|\mathcal{H}\right| < \infty$, is PAO-learnable over a partition $\mathcal{C}_{j} \in \pi_{M}(\mathcal{X}_{d}, d_{x})$ if and only if $\inf_{h \in \mathcal{H}} \mbox{FNR}_{j}(h) \leq \eta,$ where $\mbox{FNR}_{j}$ is the FNR of patients in partition $\mathcal{C}_{j}$. 
\end{thm}
\begin{proof}
See Appendix B. 
\end{proof}
While the finiteness of the hypothesis set $\mathcal{H}$ is known to the designer, one cannot determine whether such a hypothesis set can support an FNR that is less than $\eta$ since the distribution $\mathcal{D}$ is unknown to the learner. Thus, the learnability of a hypothesis set can only be determined in the learner's training phase, where the learner can infer from the training FNR estimate whether or not $\inf_{h \in \mathcal{H}} \mbox{FNR}(h) \leq \eta$. Theorem 2 also implies that solving the FNR-constrained cost minimization problem using the empirical estimates of both the cost and the FNR will lead to a solution that with probability $1-\delta$ will be within $\epsilon_{c}$ from the optimal value, and within $\epsilon$ from the FNR constraint. Thus, an algorithm $\mathcal{A}$ that solves the constrained optimization problem in (\ref{eqtn4}) ``empirically" is a ``good" learner for the hypothesis set $\mathcal{H}$. The key for the result of Theorem 2 is that if $|\mathcal{H}|<\infty,$ then the FNR and cost functions are {\it Glivenko-Cantelli} classes \cite{ref27}, for which the {\it uniform convergence} property is satisfied, i.e. every large enough training sample can be used to obtain a ``faithful" estimate of the costs and the accuracies of all the hypotheses in the set $\mathcal{H}$. We call the class of algorithms that solve optimization problem in (\ref{eqtn4}) using the empirical cost and FNR measures as {\it empirical constrained cost-minimizers} (ECCM).      

\subsection{ConfidentCare design rationale}
Based on Theorem 2 and the fact that (\ref{eqtn6}) is NP-hard, we know that ConfidentCare will comprise a heuristic partitioning algorithm $\mathcal{A}^{part}$ that obtains an approximate solution for (\ref{eqtn6}), and an ECCM learning algorithm $\mathcal{A}$ that picks a hypothesis in $\mathcal{H}$ for every partition. Since problem (\ref{eqtn6}) is NP-hard, we use a {\it Divide-and-Conquer} approach to partition the feature space: we use a simple 2-mean clustering algorithm $\mathcal{A}^{part}$ to split the a given partition in the personal feature space, and we iteratively construct a decision tree using $\mathcal{A}$ for each partition of the feature space, and then split all partitions using $\mathcal{A}^{part}$, until the algorithm $\mathcal{A}$ finds no feasible solution for (\ref{eqtn77}) for any of the existing partitions if they are to be split further. 

The algorithm $\mathcal{A}$ can be any ECCM algorithm, i.e. $\mathcal{A}$ solves the following optimization problem
\begin{equation}
\begin{aligned}
& \mathcal{A}(S^{j}_{m}) = \arg \underset{h \in \mathcal{H}}{\text{min}} 
\, \frac{1}{m_{j}} \sum_{({\bf x},y) \in S^{j}_{m}} \bar{c}\left(h({\bf x}_{s})\right) \\ 
& \text{{\it s.t.}} \,
\frac{\sum_{({\bf x},y) \in S^{j}_{m}} \mathbb{I}_{\left\{h({\bf x}_{s})\neq y, y = 1\right\}}}{\sum_{({\bf x},y) \in S^{j}_{m}} \mathbb{I}_{\left\{y = 1\right\}}} \leq \eta - \sqrt{\frac{\log\left(\left|\mathcal{H}\right|\right) + \log\left(\frac{4}{\delta}\right)}{2 m_{j}}}, 
\end{aligned} 
\label{eqtn77}
\end{equation} 
where the constraint in (\ref{eqtn77}) follows from the sample complexity of $\mathcal{H},$ which is $N^{*}\left(\epsilon, \epsilon_{c}, \delta\right) = \frac{\log\left(4\left|\mathcal{H}\right|/\delta\right)}{2 \min\{\epsilon^{2}, \epsilon^{2}_{c}\}}.$

\subsection{ConfidentCare algorithm}
The inputs to ConfidentCare algorithm can be formally given by
\begin{itemize}    
\item The size-$m$ training data set $S_{m} = ({\bf x}_{i}, y_{i})_{i \in [m]}$.
\item The FNR restriction $\eta$.
\item The confidence level $1-\delta$.
\end{itemize}

The operation of ConfidentCare relies on a clustering algorithm that is a variant of Lloyd's $K$-means clustering algorithm \cite{ref28}. However, our clustering algorithm will be restricted to splitting an input space into two clusters, thus we implement a 2-means clustering algorithm, for which we also exploit some prior information on the input space. That is, we exploit the risk assessments computed via the {\it Gail model} in order to initialize the clusters centroids \cite{ref06}-\cite{ref09}, thereby ensuring fast convergence. Let ${\bf G}: \mathcal{X}_{d} \rightarrow [0,1]$ be Gail's risk assessment function, i.e. a mapping from a patient's personal feature to a risk of developing breast cancer. Moreover, we use a distance metric that incorporates the risk assessment as computed by the Gail model in order to measure the distance between patients. The distance metric used by our algorithm is 
\[d(x,x^{'}) = \beta ||x-x^{'}|| + (1-\beta) |{\bf G}(x)-{\bf G}(x^{'})|.\]
The parameter $\beta$ quantifies how much information from the Gail model are utilized to measure the similarity between patients. Setting $\beta = 0$ is equivalent to stratifying the risk space, whereas $\beta = 1$ is equivalent to stratifying the feature space. The value of $\beta$ needs to be learned as we show later in Section V-B. 

Our clustering function, which we call $Split(\bar{\mathcal{X}}_{d}, d_{x}, \tau, \Delta)$ takes as inputs: a size-$N$ subset of the personal feature space (training set) $\bar{\mathcal{X}}_{d} = \{{\bf x}_{d}^{1},{\bf x}_{d}^{2},.\,.\,.,{\bf x}_{d}^{N}\} \subset \mathcal{X}_{d}$, a distance metric $d_{x}$, a Gail model parameter $\tau$, and a precision level $\Delta$. The function carries out the following steps:   
\begin{itemize}
\item Compute the risk assessments $\left\{{\bf G}({\bf x}^{i}_{d}, \tau)\right\}_{i=1}^{N}$ for all vectors in the (finite) input space using the Gail model. The parameter $\tau$ corresponds to the time interval over which the risk is assessed, i.e. ${\bf G}({\bf x}^{i}_{d}, \tau)$ is the probability that a patient with a feature vector ${\bf x}_{d}$ would develop a breast cancer in the next $\tau$ years. 
\item Set the initial centroids to be $\mu_{1} = {\bf x}_{d}^{i_{*}},$ where $i_{*} = \arg \min_{i} {\bf G}({\bf x}^{i}_{d}, \tau)$, and $\mu_{2} = {\bf x}_{d}^{i^{*}},$ where $i^{*} = \arg \max_{i} {\bf G}({\bf x}^{i}_{d}, \tau)$.
\item Create two empty sets $\mathcal{C}_{1}$ and $\mathcal{C}_{2}$, which represent the members of each cluster.
\item Until convergence (where the stopping criterion is determined by $\Delta$), repeat the following: assign every vector ${\bf x}^{i}_{d}$ to $\mathcal{C}_{1}$ if $d_{x}({\bf x}^{i}_{d}, \mu_{1}) < d_{x}({\bf x}^{i}_{d}, \mu_{2}),$ and assign it to $\mathcal{C}_{2}$ otherwise. Update the clusters' centroids as follows 
\[\mu_{j} = \frac{1}{\left|\mathcal{C}_{j}\right|} \sum_{i=1}^{N}\mathbb{I}_{{\bf x}^{i}_{d} \in \mathcal{C}_{j}} {\bf x}^{i}_{d}, j \in \{1,2\}.\]   
\item Return the clusters' centroids $\mu_{1}$ and $\mu_{2}$.
\end{itemize}
The rationale behind selecting the initial centroids as being the feature vectors with maximum and minimum risk assessments is that those two patients' features are more likely to end up residing in different clusters. A detailed pseudocode for the clustering function is given in Algorithm 1. As we will show later, ConfidentCare will utilize this function to iteratively partition the personal feature space.
\RestyleAlgo{boxruled}
\LinesNumbered
\begin{algorithm}
 \algorithmicrequire{ A set $N$ training vectors $\bar{\mathcal{X}}_{d}$, $K > M$, a distance metric $d_{x}$, a Gail model parameter $\tau$, and a precision level $\Delta$.} \\
 \algorithmicensure{ Two centroids $\mu_{1}$ and $\mu_{2}$}\;
 Initialize $D_{-1} = 1$, $D_{0} = 0$, $k = 0$, and
$\mu_{1} = {\bf x}_{d}^{i_{*}}, i_{*} = \arg \min_{i} {\bf G}({\bf x}^{i}_{d}, \tau),$ \;
$\mu_{2} = {\bf x}_{d}^{i^{*}}, i^{*} = \arg \max_{i} {\bf G}({\bf x}^{i}_{d}, \tau)$ \;    
$\mathcal{C}_{1} = \emptyset, \mathcal{C}_{2} = \emptyset$  \;
 \While{$\frac{D_{k-1}-D_{k}}{D_{k}} > \Delta$}{
  	$\mathcal{C}_{1} = \left\{{\bf x}^{i}_{d}\left|\forall {\bf x}^{i}_{d} \in \mathcal{X}_{d}, d_{x}({\bf x}^{i}_{d}, \mu_{1}) < d_{x}({\bf x}^{i}_{d}, \mu_{2})\right.\right\}$  \;
$\mathcal{C}_{2} = \bar{\mathcal{X}}_{d}/\mathcal{C}_{1}$\;
$\mu_{1} = \frac{1}{\left|\mathcal{C}_{1}\right|} \sum_{i=1}^{N}\mathbb{I}_{{\bf x}^{i}_{d} \in \mathcal{C}_{1}} {\bf x}^{i}_{d}$\;
$\mu_{2} = \frac{1}{\left|\mathcal{C}_{2}\right|} \sum_{i=1}^{N}\mathbb{I}_{{\bf x}^{i}_{d} \in \mathcal{C}_{2}} {\bf x}^{i}_{d}$\;
  Set $k \leftarrow k+1$\;
	Compute the {\it 2-means objective function} $D_{k} = \frac{1}{N}\sum_{j=1}^{2}\sum_{i=1}^{N} \mathbb{I}_{{\bf x}^{i}_{d} \in \mathcal{C}_{j}}d_{x}({\bf x}^{i}_{d}, \mu_{j})$\;
 
 }
 \caption{$Split(\bar{\mathcal{X}}_{d}, d_{x}, \tau, \Delta)$.}
\end{algorithm}

For a given feature space partitioning, ConfidentCare builds an active classifier that emulates a ``virtual CPG" for the set of patients within the partition. Designing the active classifier is equivalent to: following an inductive bias approach in which a specific hypothesis class $\mathcal{H}$ is picked, and designing an algorithm $\mathcal{A}$ that takes the training set $S_{m}$ as an input and picks the ``best" hypothesis in $\mathcal{H}$, i.e. $\mathcal{A}(S_{m}) \in \mathcal{H}$. 

Adopting decision trees as a hypothesis set is advantageous since such a classifier is widely used and easily interpretable for medical applications \cite{ref29}-\cite{ref31}. As shown in Fig. \ref{fgtree}, ConfidentCare will associate a decision tree active classifier with every partition of the personal feature space. Such a tree represents the policy to follow with patients who belong to that partition; what tests to recommend and how to map the BI-RADS scores resulting from one test to a new test recommendation or a diagnostic decision. 

Learning the optimal decision tree $h^{*} \in \mathcal{H}$ is known to be an NP-hard problem \cite{ref32}, thus we resort to greedy algorithm $\mathcal{A},$ which we call the confidence-based Cost-sensitive decision tree induction algorithm ($ConfidentTree$). The main idea of $ConfidentTree$ is to select tests (nodes of the tree) in a greedy manner by using a splitting rule that operates as follows: in each step, label the leaves that come out of each possible test such that the pessimistic estimate for the FNR (given the confidence level $1-\delta$) is less than $\eta$, and then pick the test that maximizes the ratio between the information gain and the test cost. After growing such a tree, we apply post-pruning based on confidence intervals of error estimates \cite{ref33}. If there is no possible labeling of the tree leaves that satisfy the FNR requirements, the algorithm reports the infeasibility of the FNR and confidence levels set by the clinician given the training set provided to the program. More precisely, the algorithm $ConfidentTree(S_{m},\pi_{M}(\mathcal{X}_{d}, d_{x}), j, \eta, 1-\delta)$ takes the following inputs:
\begin{itemize}  
\item The size-$m$ training set $S_{m}$.
\item The personal feature space partitioning $\pi_{M}(\mathcal{X}_{d}, d_{x})$.
\item The index $j$ of the partition for which we are designing the active classifier.
\item The FNR constraint $\eta$.
\item The confidence level $1-\delta$.
\end{itemize}
Given these inputs, the algorithm then executes the following steps:
\begin{itemize}
\item Extract the training instances that belong to partition $\mathcal{C}_{j}$.
\item Grow a decision tree with the nodes being the screening tests in $\mathcal{T}$. The edges are the BI-RADS scores with the following thresholds: BI-RADS $<$ 3, BI-RADS $\in \{3,4\},$ and BI-RADS $>$ 4. This classification is based on domain knowledge \cite{ref19}; the first category corresponds to a probably negative diagnosis, the second corresponds to a suspicious outcome, whereas the third corresponds to a probably malignant tumor.   
\item While growing the tree, the splitting rule is as follows: for each test, label the leaves such that the pessimistic estimate (see \cite{ref33} for confidence interval and error estimates in the C4.5 algorithm) for the FNR is equal to $\eta$, and then compute the cost function for each test, and select the test that maximizes the ratio between the information gain and the cost function.
\item Apply post-pruning based on confidence intervals of the error estimates as in the C4.5 algorithm \cite{ref33}. This step is carried out in order to avoid overfitting.
\item If the pessimistic estimate for the FNR exceeds $\eta$, report the infeasibility of constructing a decision tree with the given FNR and confidence requirements. 
\end{itemize}
A detailed pseudocode for $ConfidentTree$ is given in Algorithm 2. ConfidentCare invokes this algorithm whenever the personal feature space is partitioned, and the active classifiers need to be constructed. 
\RestyleAlgo{boxruled}
\LinesNumbered
\begin{algorithm}
 \algorithmicrequire{ A set of training instances $S_{m}$, a partitioning $Part_{M}(\mathcal{X}_{d}, d_{x})$, a partition index $j$, Maximum tolerable FNR $\eta$, and confidence level $1-\delta$.} \\
 \algorithmicensure{ A cost-sensitive decision-tree $h_{j}$ that can be used as an active classifier for partition $\mathcal{C}_{j}$.}\;
Let $B_{1}$ be the event that BI-RADS $<$ 3, $B_{2}$ be that BI-RADS $\in \{3,4\},$ and $B_{3}$ be BI-RADS $>$ 4 \;
Extract the training set that belong to the targeted partition $S^{j}_{m} = \left\{({\bf x}_{i}, y_{i})\left|\forall i \in [m], {\bf x}_{i,d} \in \mathcal{C}_{j}\right.\right\}$\;
For each test, label the leaves attached to edges $B_{1}$, $B_{2}$, and $B_{3}$ such that the empirical FNR is less than the solution of the following equation for $\hat{F}$
\[\eta = \frac{\hat{F} + \frac{Q^{-1}(\delta)}{2n} + Q^{-1}(\delta)\sqrt{\frac{\hat{F}}{n} - \frac{\hat{F}^{2}}{n} + \frac{Q^{-1}(\delta)^{2}}{4n^{2}}}}{1+\frac{Q^{-1}(\delta)^{2}}{n}},\]
where $Q(.)$ is the Q-function and $n$ is the number of training instances covered by the leaf for which the classification is 1. \; Given this labeling, let $\hat{F}_{p}$ be the empirical value of the false positive rate, then pick the test $s \in \mathcal{T}$ that maximizes $\frac{I(s;S^{j}_{m})}{\gamma \hat{F}_{p} + (1-\gamma) \bar{c}_{s}}$, where $I(x;y)$ is the mutual information between $x$ and $y$. \;
Apply post-pruning using confidence intervals for error estimates: a node is pruned if the error estimate of its induced sub-tree i s lower than error estimate of the node.
 \caption{$ConfidentTree(S_{m},Part_{M}(\mathcal{X}_{d}, d_{x}), j, \eta, 1-\delta)$}
\end{algorithm}

ConfidentCare uses the modules $ConfidentTree$ and $Split$ in order to iteratively partition the feature space and construct active classifiers for each partition. ConfidentCare runs in two stages: the offline policy computation stage, and the policy execution stage. In the offline policy computation stage, the following steps are carried out:
\begin{enumerate}
\item Use the $Split$ function to split all current partitions of the personal feature space.  
\item Use the $ConfidentTree$ to create new active classifiers for the split partitions, if constructing a decision tree for a specific partition is infeasible, stop splitting this partition, otherwise go to step (1).
\end{enumerate} 
After computing the policy, ConfidentCare handles the incoming patients in the policy execution stage as follows:
\begin{enumerate}
\item Observe the personal features of the incoming patient, measure the distance between her feature vector and the centroids of the learned partitions, and associate her with the closest partition and the associated active classifier.   
\item Apply active classification to the patient. After each test outcome, ConfidentCare prompts a recommended test (the next node in the decision tree), and an intermediate diagnosis together with an associated confidence interval. The clinician and the patient will then decide whether or not to proceed and take the next test.  
\end{enumerate}   
The pseudocode for ConfidentCare in both the offline and online modes is given in Algorithm 3. In the following Theorem, we show that the greedy ConfidentCare algorithm can guarantee a reasonable performance.

\RestyleAlgo{boxruled}
\LinesNumbered
\begin{algorithm}
 \algorithmicrequire{A training set $S_{m}$, required confidence level $\delta$, and FNR constraint $\eta$.} \\
 \algorithmicensure{ A sequence of recommendations, intermediate diagnoses with confidence intervals, and a final diagnosis}\;
{\bf Offline policy computation stage:} \;
	Initialize $M = \infty, q = 0$\;
	Initialize ${\bf \mu} = \emptyset$ (set of centroids of the personal feature space) \;
	Hyper-parameters $\tau$, $\gamma$, and $\Delta$ can be tuned through a validation set\; 
  \While{$q \neq M$}{
	$M = \left|{\bf \mu}\right|$ \;
	Create a partitioning $Part(\mathcal{X}_{d}, d_{x})$ based on the centroids in ${\bf \mu}$ \;
	{\bf For} $j = 1$ to $M$ \;
	  \,\,\,\,\,   ${\bf \mu} \rightarrow Split(\mathcal{X}_{d},d_{x},\tau,\Delta)$\;
    \,\,\,\,\,   $h_{j} = ConfidentTree(S_{m},Part_{M}(\mathcal{X}_{d}, d_{x}), j, \eta, 1-\delta)$ \;
		\,\,\,\,\,	 If $h_{j}$ is infeasible: $q \leftarrow q+1$ \;
  {\bf EndFor} 
 }
{\bf Policy execution stage:} \;
For the incoming patient $i$, find the partition it belongs to by computing the distance $d_{x}({\bf x}_{i,d}, \mu_{j})$ for every partition $\mathcal{C}_{j}$, and associate it with the partition $j^{*}$ that gives the minimum distance \;
Use classifier $h_{j^{*}}$ to recommend tests and issue diagnoses
 \caption{$ConfidentCare\left(S_{m}, \delta, \eta\right)$.}
\end{algorithm}	
\begin{figure*}[t!]
    \centering
    \includegraphics[width=6.5 in]{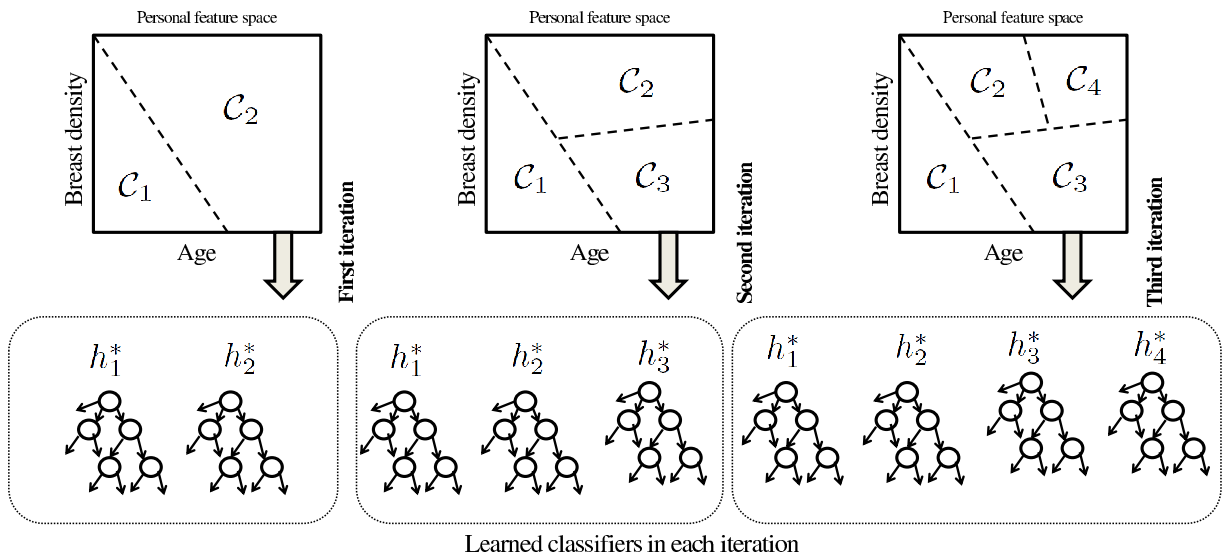}     
    \caption{Demonstration for the operation of ConfidentCare iterative algorithm. In each iteration, the personal feature space is split and a decision tree is learned for the newly emerging partition of the space.}
		\label{fgiter}
\end{figure*} 
Fig. \ref{fgiter} demonstrates the operation of the iterative algorithm; in each iteration, partitions are split as long as a decision tree for the new partitions are feasible, and the corresponding decision trees are learned. The end result is a set of decision trees for the different partitions, representing different policies to be followed for every class of patients. Following the CPGs correspond to having a single decision tree for the entire personal feature space, which may consistently perform poorly over specific partitions of the feature space, i.e. specific subgroups of patients. 

\section{Experiments}
In this section, we demonstrate the operation of ConfidentCare by applying it to a real-world dataset for breast cancer patients. Moreover, we evaluate the performance of ConfidentCare and the added value of personalization by comparing it with CPGs, and policies that are designed in a ``one-size-fits-all" fashion. We start by describing the dataset used in all the experiments in the following subsection. 

\subsection{Data Description}
We conduct all the experiments in this section using a de-identified dataset of 25,594 individuals who underwent screening via mammograms, MRIs and ultrasound at the UCLA medical center. The features associated with each individual are: age, breast density, ethnicity, gender, family history, age at menarche, age at the first child birth and hormonal history. Each individual has underwent at least one of three screening tests: a mammogram (MG), an MRI, an ultrasound (US), or a combination of those. With each test taken, a BI-RADS score is associated. Table \ref{Dataset} shows the entries of the dataset and the features associated with every patient. The dataset is labeled by $0$ for patients who have a negative diagnosis, and $1$ for patients who have a positive diagnosis (malignant tumor). The dataset exhibits a sampling bias in the sense that patients who took a MG are much more than those who took an US or an MRI. Moreover, most patients exhibited negative test results. Table \ref{Dataset2} lists the percentages of patients who took each screening test, and the percentage of patients with positive diagnoses. All features were converted into numerical values and normalized to take values between 0 and 1. The normalized monetary costs for MG, US, and MRI where set to 0.1, 0.2 and 0.7 respectively, and $\gamma$ is set to 0.5. In the following subsection, we demonstrate the operation of ConfidentCare.         

\begin{table*}
\captionsetup{font= small}
\caption{De-identified breast cancer screening tests dataset}
\begin{center}
 \centering
    \begin{tabular}{||M{1.4cm}||M{0.75cm}|M{2.25cm}|M{1cm}|M{0.75cm}|M{1.75cm}|M{1.25cm}|M{1.5cm}|M{1.5cm}|M{1.5cm}|M{1.5cm}|} \hline
        {\bf Patient ID} & {\bf Age} & {\bf Breast density} & {\bf Ethnicity} & {\bf Gender} & {\bf Family history} & {\bf Hormonal history} &  {\bf Mammogram BI-RADS score} & {\bf MRI BI-RADS score} & {\bf Ultrasound BI-RADS score} \\ \hline \hline
				    1 & 71 & Almost entirely fat ($<$25$\%$) & U & F & Maternal Aunt & - & 1 & - & - \\ \hline
    2 & 72 & Almost entirely fat ($<$25$\%$) & U & F & Maternal Cousin & Estrogen & 2 & 1 & -   \\ \hline
    3 & 60 & Heterogeneously dense (51$\%$ - 75$\%$) & B & F & - & Estrogen & 2 & - & -  \\ \hline
		4 & 66 & Almost entirely fat ($<$25$\%$) & W & F & Sister & - & 1 & - & -   \\ \hline
		5 & 56 & Heterogeneously dense (51$\%$ - 75$\%$) & W & F & - & - & 1 & - & -   \\ \hline
		. & . & . & . & . & . & . & . & . & . \\ 
		. & . & . & . & . & . & . & . & . & . \\ \hline
		11,733 & 39 & Heterogeneously dense (51$\%$ - 75$\%$) & A & F & - & - & 2 & 1 & -  \\ \hline
		. & . & . & . & . & . & . & . & . & . \\
		. & . & . & . & . & . & . & . & . & . \\  \hline
		25,594 & 67 & Heterogeneously dense (51$\%$ - 75$\%$) & W & F & Mother & Tamoxifen & 2 & 2 & 1  \\
    \hline
    \end{tabular}
\end{center}
\label{Dataset}
\end{table*}
\begin{table}[t!]
\captionsetup{font= small}
\caption{Statistics for the dataset involved in the experiments}
\begin{center}
 \centering
    \begin{tabular}{||M{4cm}||M{3.75cm}|} \hline
        {\bf Category} & {\bf Percentage} \\ \hline \hline
				    MG BI-RADS & 93.39$\%$.   \\ \hline
						MRI BI-RADS & 2.75$\%$.  \\ \hline
						US BI-RADS & 9.21$\%$. \\ \hline
						Patients with malignant tumor & 8.33$\%$. \\  
    \hline
    \end{tabular}
\end{center}
\label{Dataset2}
\end{table}  
\subsection{Learning the distance metric}
\begin{figure}[t!]
    \centering
    \includegraphics[width=3.5 in]{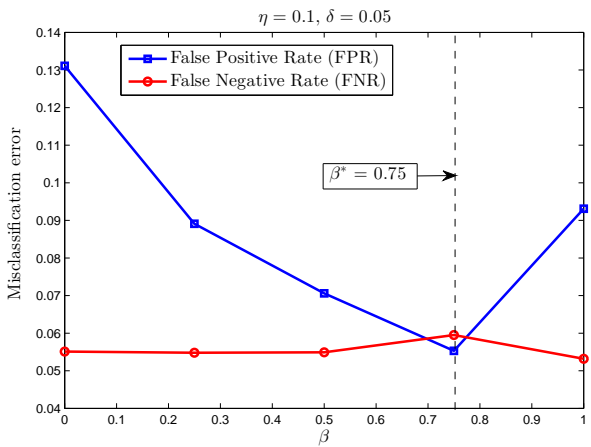}     
    \caption{Optimal selection for the distance metric parameter $\beta$ for $\eta = 0.1$ and $\eta = 0.1$.}
		\label{sm1}
\end{figure}
\begin{figure}[t!]
    \centering
    \includegraphics[width=3.5 in]{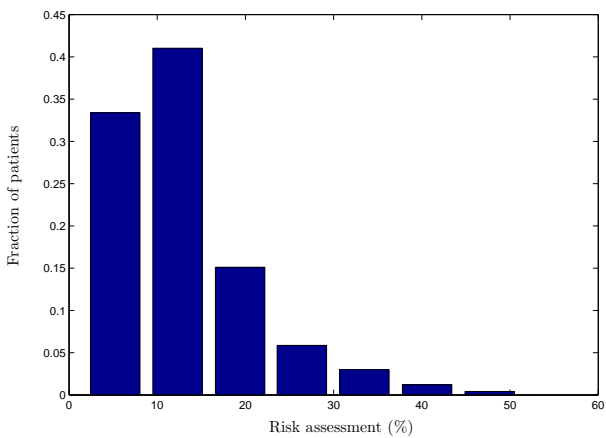}     
    \caption{Histogram for the Gail risk assessments for patients in the dataset.}
		\label{sm2}
\end{figure}
Recall from Section IV that clustering of the patients' personal  feature space was carried out using a distance metric that combines both the feature values and the risk assessments as computed by the Gail risk model using the parameter $\beta$. Setting the parameter $\beta = 0$ corresponds to risk stratification, whereas setting $\beta = 1$ corresponds to stratifying the personal feature space while disregarding the prior information provided by the Gail model. Since the Gail model does not incorporate all the patients features (e.g. family history), one expects that the best choice of $\beta$ will be between 0 and 1, for that both the personal features and the risk assessments of the patients contains (non-redundant) information about patients' similarity. As shown in Fig. \ref{sm1}, for an FNR constraint of $\eta = 0.1$ and confidence parameter of $\delta = 0.05$, we found that $\beta = 0.75$ is the best choice of the distance metric since it maximizes the system's accuracy (FNR and FPR). This means that for $\eta = 0.1$ and $\delta = 0.05$, it is better to incorporate more information from the personal features than from the risk assessment. Our interpretation for such a result is that since most of the patients in the dataset have a low to average risks as shown in the histogram plotted in Fig. \ref{sm2}, the information contained in the Gail risk assessment is not enough to differentiate between patients and bundle them into clusters. For the rest of this section, we use the value $\beta = 0.75$ when running ConfidentCare in all the experiments.       

\subsection{ConfidentCare operation}
\begin{figure}[t!]
    \centering
    \includegraphics[width=3.5 in]{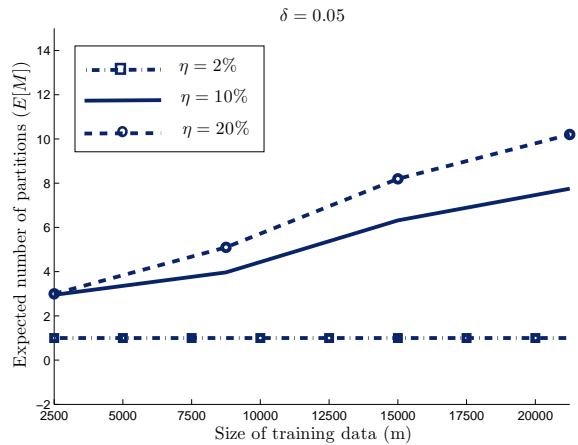}     
    \caption{The expected number of partitions (clusters) of the personal feature space versus the size of the training set .}
		\label{sm3}
\end{figure}
\begin{figure}[t!]
    \centering
    \includegraphics[width=3.5 in]{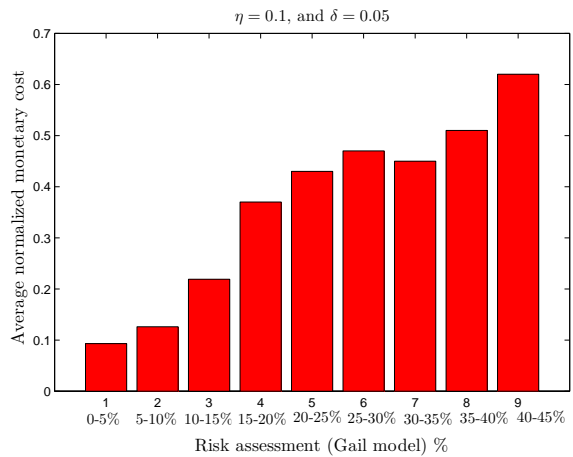}     
    \caption{Average normalized monetary cost endured by ConfidentCare for patients with different risk assessments.}
		\label{sm33}
\end{figure}
\begin{figure}[t!]
    \centering
    \includegraphics[width=3.5 in]{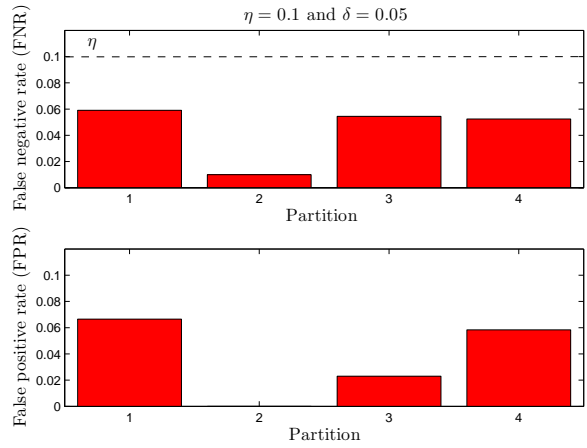}     
    \caption{FNR and FPR of ConfidentCare for different partitions of the personal feature space.}
		\label{sm4}
\end{figure}
\begin{figure*}[t!]
    \centering
    \includegraphics[width=6.5 in]{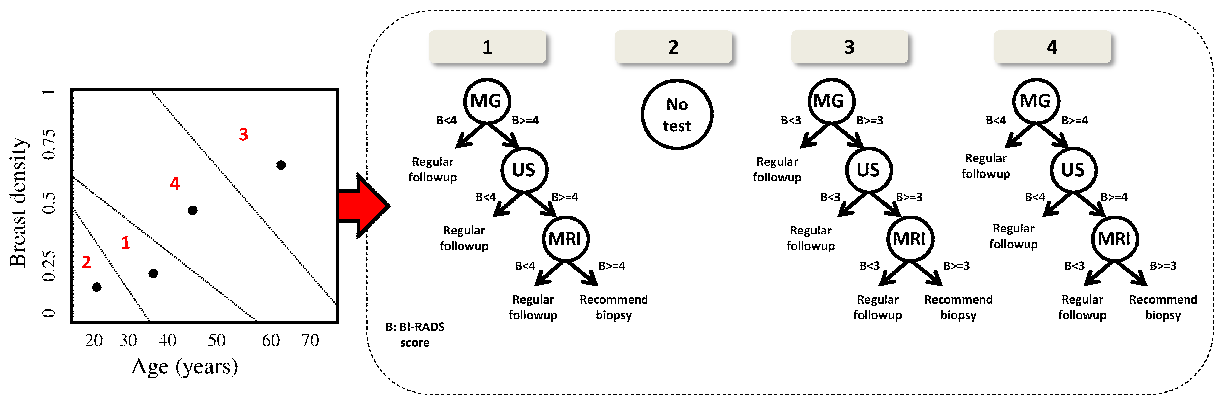}     
    \caption{The personal feature space partitions and the corresponding screening policy.}
		\label{sm6}
\end{figure*}
\begin{figure}[t!]
    \centering
    \includegraphics[width=3.5 in]{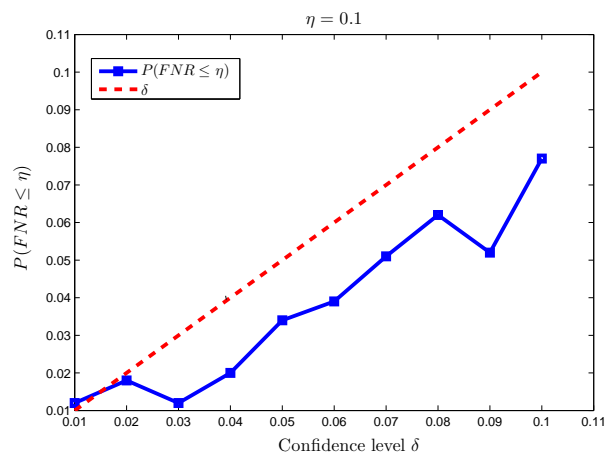}     
    \caption{The probability that the FNR of ConfidentCare is below $\eta$ versus the confidence parameter $\delta$.}
		\label{sm5}
\end{figure}
In this subsection, we investigate the operation of ConfidentCare in terms of clustering and policy construction, endured monetary costs, and accuracy. As we can see in Fig. \ref{sm3}, ConfidentCare can (on average) discover more subgroups of patients for whom it can construct a screening policy with the desired confidence level as the size of the training data increases. This is intuitive since the more training instances are available, the more granular are the partitions that can be formed by the algorithm over the personal feature space. Note that for different settings for the constraint $\eta$, the possible levels of stratification are different. For a fixed size of the training data, as the FNR constraint becomes tighter, the level of personalization decreases. For instance, we can see in Fig. \ref{sm3} that the expected number of partitions for $\eta = 0.2$ is greater than that for $\eta = 0.1$, whereas for $\eta = 0.02$ the system can never find any feasible partitioning of the feature space regardless of the size of the training data.  

Fig. \ref{sm33} shows the average (normalized) monetary costs endured by ConfidentCare for patients with different risk assessments. As the risk level increases, the costs increase consequently since ConfidentCaare would recommend more tests (including the expensive MRI test) to patients with high level of risk for developing breast cancer. This is again follows from the impact of personalization: only patients who need the screening tests are recommended to take it, for that the screening policy behaves differently for different patient subgroups.   

In Fig. \ref{sm4}, we plot the FNR and FPR with respect to every partition constructed by the algorithm in a specific realization of ConfidentCare which was able to discover 4 partitions. It is clear that the FNR satisfies the constraint of $\eta = 0.1$ for all partitions. The FPR for different partitions, for instance we can see that partition 2 has a FPR of 0, whereas other partitions have a non-zero FPR. In Fig. \ref{sm6}, we show the partitions (in a 2D subspace of the original personal feature space) and the constructed policy corresponding to each cluster. It can be seen that patients who are young in age and have low breast density are recommended to take no tests, whereas other subgroups are recommended to take a MG test. We also note that the policy is more aggressive for patients with high breast density, i.e. for partition 3, a relatively low BI-RADS score from a MG can still lead to a recommendation for an addition US or an MRI, whereas for other subgroups the policy is more conservative in terms of recommending additional screening tests. This is because detecting a tumor is more difficult for patients with high breast density.  

Note that Fig. \ref{sm4} represents just a single realization of ConfidentCare, and thus it does not reveal the amount of confidence we have in the algorithm being satisfying the FNR constraint with a high probability. In order to verify the confidence level in the policy constructed by ConfidentCare, we run the algorithm for 100 runs and see the fraction of time where the FNR in the testing set for any partition exceeds the threshold $\eta$. It can be seen that this is bounded by the specified confidence level $\delta$.     

\subsection{ConfidentCare performance evaluation}
\begin{figure}[t!]
    \centering
    \includegraphics[width=3.5 in]{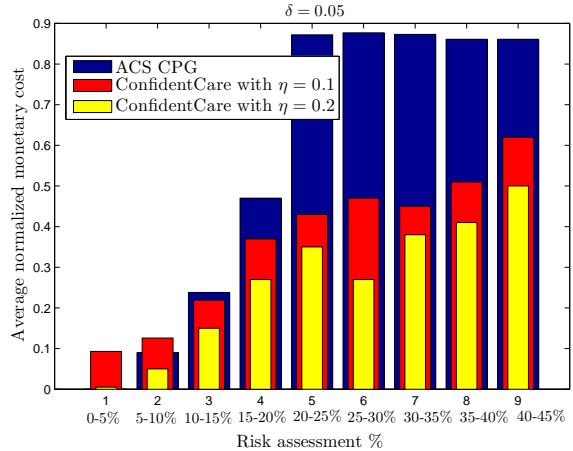}     
    \caption{Average normalized monetary cost versus risk assessment for ConfidentCare and the ACS guidelines.}
		\label{sm7}
\end{figure}
\begin{figure}[t!]
    \centering
    \includegraphics[width=3.5 in]{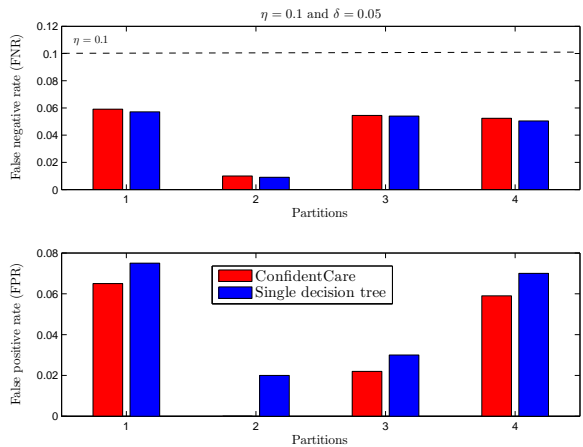}     
    \caption{FNR and FPR of ConfidentCare and a single decision tree of screening tests.}
		\label{sm8}
\end{figure}
We compare the performance of ConfidentCare with that of the current clinical guidelines in order to assess the value of personalization in terms of cost-efficiency. We compare the monetary cost of ConfidentCare with that of the American Cancer Society (ACS) screening guidelines issued in 2015 \cite{ref33X}. The reason for selecting this specific CPG is that it already applies a coarse form of risk stratification: low, average and high risk women are recommended to take different sets of tests. In Fig. \ref{sm7}, we plot the distribution of the normalized monetary cost of ConfidentCare together with that of the ACS over different levels of risk. It is clear that ConfidentCare can save a significant amount of screening costs since it supports a finer stratification of the patients, and thus recommends screening tests only to patients who need them based on both their features the outcomes of the previous tests that they may have taken. The comparison in Fig. \ref{sm7} is indeed subject to the selection of $\eta$ and $\delta$. The more we relax the FNR and confidence constraints, the more savings we attain for the monetary costs. The cost-efficiency of ConfidentCare depends on the selection of $\eta$ and $\delta$, which can be set by clinicians or institutions, and based on which the added value of a personalized system can be assessed.       

Finally, we compare the accuracy of ConfidentCare with that of a single decision tree of tests that is designed in a ``one-size-fits-all" fashion. In particular, we build a tree of tests using the conventional C4.5 algorithm \cite{ref33}, and then compare its performance with that of ConfidentCare with respect to every partition found by ConfidentCare. From Fig. \ref{sm8}, we can see that for the same realization illustrated in Fig. \ref{sm4} and \ref{sm6}, both approaches have a comparable FNR, but ConfidentCare outperforms a single decision tree in terms of the FPR for all the 4 partitions. This is because ConfidentCare deals differently with women belonging to different subgroups as shown in Fig. \ref{sm6}, i.e. for instance women in partition 2 are not recommended to take any tests. In other words, ConfidentCare avoids recommending unnecessary tests, which reduces the rate of false positives. The average values of the FNR and FPR for 50 runs of ConfidentCare and a single decision tree are reported in Table \ref{perftable}, where a gain of $31.91\%$ with respect to the FPR is reported.       
\begin{table}[t!]
\captionsetup{font= small}
\caption{FNR and FPR for ConfidentCare (with $\eta = 0.1$ and $\delta = 0.05$) and a single C4.5 decision tree}
\begin{center}
 \centering
    \begin{tabular}{||M{3.75cm}||M{1.75cm}||M{1.75cm}|} \hline
        {\bf Algorithm} & {\bf FNR} & {\bf FPR} \\ \hline \hline
				    Single C4.5 decision tree & 0.0501 & 0.0488.   \\ \hline
						ConfidentCare & 0.0512 & 0.037.  \\  
    \hline
    \end{tabular}
\end{center}
\label{perftable}
\end{table} 
\subsection{Discussions and future work}
\begin{figure}[t!]
    \centering
    \includegraphics[width=3.5 in]{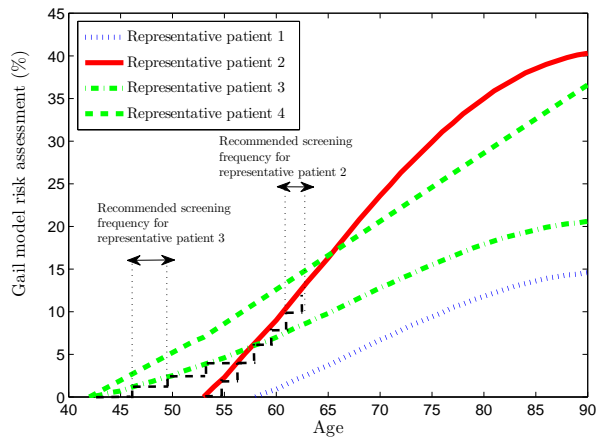}     
    \caption{Risk assessment over time for the representative patients (centroids) constructed by ConfidentCare.}
		\label{sm9}
\end{figure}
The screening policy considered in this paper was concerned with managing the short-term screening procedure, i.e. the policy was recommending a sequence of screening tests for the patient based on the outcomes of those tests, and such tests are expected to be taken in a relatively short time interval. Our framework can be extended to design policies that are concerned with the long-term patient outcomes, and are capable of not only recommend tests to the patient, but also recommend the frequency with which screening should be carried out for different subgroups of patients. To see how our framework can be extended to handle such a setting, we plot the risk of developing breast cancer over time for the representative agents (centroids) of 4 clusters constructed in one realization of the algorithm in Fig. \ref{sm9}. Each cluster exhibits a different rate of risk growth over time, i.e. for instance while clusters 3 and 4 in Fig. \ref{sm9} comprise women of almost the same age, patients in cluster 3 develop a risk for breast cancer more quickly than patients in cluster 4 due to other factors (e.g. family history). Thus, ConfidentCare can be modified to not only recommend a sequence of tests to patients in different clusters, but also to compute the optimal frequency of screening (steps over time for which the patient need to be regularly screened) that would maximize a long-term objective function. Intuitively, the frequency of screening would depend on the slope of the risk assessment over time, i.e. clusters with steeper slopes would demand more frequent screening. Our framework is well suited to capture such a setting, and the ConfidentCare algorithm can be modified to construct a screening policy that maximizes long-term outcomes with high levels of confidence.     

\section{Conclusions}
In this paper, we presented ConfidentCare: a clinical decision support system that learns a personalized screening policy from the electronic health data record. ConfidentCare operates by stratifying the space of a woman? features and learning cost-effective, and accurate screening policy that is tailored to those features and are accurate with a high-level of confidence. ConfidentCare algorithm iteratively stratifies the patients' feature space into disjoint clusters and learns active classifiers associated with each cluster. We have shown that the proposed algorithm has the potential of improving the cost efficiency and accuracy of the screening process compared to current clinical practice guidelines, and state-of-the-art algorithms that do not consider personalization.   

\section*{Acknowledgment}
\addcontentsline{toc}{section}{Acknowledgment}
We would like to thank Dr. Camelia Davtyan (Ronald Reagan UCLA Medical Center) for her valuable help and precious comments on the medical aspects of the paper. We also thank Dr. William Hoiles (UCLA) for the valuable discussions and comments that we had with him on this paper.

\end{document}